\documentclass[runningheads]{llncs}
\usepackage[T1]{fontenc}
\usepackage{amsfonts}
\usepackage{graphicx}
\usepackage{amsmath}
\usepackage{amssymb}
\usepackage{algorithm}
\usepackage{algpseudocode}
\usepackage{booktabs}
\usepackage{multirow}
\usepackage{tikz}
\usetikzlibrary{shapes.geometric, arrows.meta}

\begin{document}
\title{Provable Limits and Certified Deferral for Verbalized
Uncertainty in Small Language Models}
\titlerunning{Verbalized Uncertainty in Small LMs}

\author{Jianru Shen\orcidID{0009-0000-3546-9616}}
\authorrunning{J. Shen}
\institute{University of Montana, Missoula, MT 59812, USA \\
}

\maketitle              
\begin{abstract}
Small open-weight language models increasingly run in private,
offline, and cost-sensitive settings, where the key deployment
question is not only what a model answers but when it should defer
to a human. We study whether verbalized confidence can support
risk-controlled deferral, evaluating eleven instruction-tuned models
from three families, 0.5B to 14B parameters, on ARC-Challenge and
TruthfulQA with 25,168 local predictions. Three theoretical results
delimit what calibration can provide: strictly monotone calibration
preserves the risk--coverage frontier and error-detection AUROC;
temperature scaling cannot calibrate models whose confidence stays
above one half while accuracy falls below it; and a Clopper--Pearson
procedure converts a 200-question calibration set into a
finite-sample risk certificate under an i.i.d.\ deployment
assumption. Empirically, eight of 22 model--task pairs hit the
temperature-scaling infeasibility floor within one percentage point
of the predicted bound. Platt scaling reduces ECE to as low as 0.02,
yet certified autonomy at a 20\% risk budget is granted to only
three model--task pairs and to none at 10\%. We also identify and
repair an answer-ordering artifact in the multiple-choice form of
TruthfulQA. Calibration gives confidence semantics; certified
deferral determines when small models are safe to use.

\end{abstract}
%
%
%

\section{Introduction}\label{sec:intro}

Small open-weight language models increasingly run where frontier
models cannot: on private infrastructure for confidential data, on
edge devices without connectivity, and in cost-sensitive pipelines
that cannot afford API calls per query. Precisely in this regime the
models are weakest, so the central deployment question is not only
what a model answers but when it should not answer at all. A
human-in-the-loop system needs a trustworthy confidence per
prediction, a threshold separating autonomous answers from
deferrals, and a reason to believe the resulting error rate stays
within a stated budget.

Verbalized confidence, an integer the model states alongside its
answer, is the only confidence signal available through a bare text
interface, and prior work shows it correlates with correctness for
large models~\cite{kadavath2022,tian2023}. Whether it can support
risk-controlled deferral at 0.5B to 14B parameters is unknown.
Post-hoc calibration is well studied for
classifiers~\cite{guo2017calibration}, and selective prediction
formalizes abstention~\cite{geifman2017}, but three gaps separate
these tools from a deployable system: calibration quality across
small-model families and task types is uncharted, the relation
between calibration and threshold-based deferral is usually left
implicit, and empirical error rates are reported where deployment
needs guarantees.

We close these gaps with a framework coupling verbalized confidence
extraction, post-hoc calibration, and certified threshold selection,
characterized across eleven models from three families on
ARC-Challenge and TruthfulQA. Along the way we identify and repair
an answer-ordering artifact in the multiple-choice form of
TruthfulQA that inflates letter-scored accuracy for position-biased
models.

Our contributions are threefold. First, a risk-controlled deferral
formulation for verbalized confidence, with propositions proving
that strictly monotone calibration preserves the risk--coverage
frontier and a corollary giving an exact infeasibility bound for
temperature scaling that our measurements meet to within one
percentage point. Second, a certified threshold-selection procedure
based on Clopper--Pearson bounds with a union-bound correction for
data-dependent selection. Third, a systematic study of eleven
open-weight models from 0.5B to 14B across three families on two
contrasting tasks, together with an answer-ordering audit for
TruthfulQA, threshold-transfer and elicitation-format ablations, and
transparent per-example deferral decisions that translate the theory
into deployment guidance.

\section{Related Work}\label{sec:related}

\paragraph{Confidence estimation for language models.}
Token probabilities and semantic uncertainty provide direct
uncertainty signals~\cite{kuhn2023,wang2023,farquhar2024detecting}
but require white-box access or multiple generations that deployed
text interfaces rarely afford. Verbalized confidence instead asks
the model to state its own uncertainty: large models can assess
their answers when prompted~\cite{kadavath2022}, prompting and
fine-tuning improve stated
confidence~\cite{tian2023,lin2022teaching}, and surveys document
both its promise and its
overconfidence~\cite{xiong2024,geng2024survey}. Much of this literature studies larger models or commercial APIs;
we ask whether the signal is usable at 0.5B to 14B, the scale of
local deployment.

\paragraph{Post-hoc calibration.}
Temperature scaling~\cite{guo2017calibration}, Platt
scaling~\cite{platt1999}, and isotonic
regression~\cite{zadrozny2002} are the standard post-hoc maps, with
binned ECE the standard quality measure~\cite{naeini2015}, and are
usually compared by calibration error alone. We compare them by
what they preserve for downstream deferral: strictly monotone maps
leave the risk--coverage frontier unchanged, isotonic ties coarsen
the achievable coverage grid, and Platt's bias term is exactly the
degree of freedom that separates feasible from infeasible
calibration on overconfident small models.

\paragraph{Selective prediction and learning to defer.}
Selective classification equips a model with a reject
option~\cite{geifman2017selectivenet}, with guaranteed-risk variants
bounding the error of the retained set~\cite{geifman2017}; learning
to defer trains the policy jointly against a human
expert~\cite{mozannar2020,madras2018}, requiring expert labels and
retraining. Conformal methods give finite-sample coverage guarantees
for prediction sets~\cite{angelopoulos2023conformal} and extend to
language generation~\cite{quach2024conformal}; our certificate
instead bounds the conditional risk of a threshold-based
answer/defer policy. Our setting is post hoc: the model is fixed,
only a small calibration set is available, and the threshold must
carry a finite-sample guarantee~\cite{geifman2017}, issued per task,
since thresholds do not transfer across tasks.

\section{Method}\label{sec:method}
Figure~\ref{fig:overview} summarizes the certified deferral
pipeline. Our framework has three components: verbalized confidence
extraction, post-hoc calibration, and risk-controlled threshold
selection with a finite-sample certificate. We formalize each in
turn.

\begin{figure}[t]
\centering
\includegraphics[width=\textwidth]{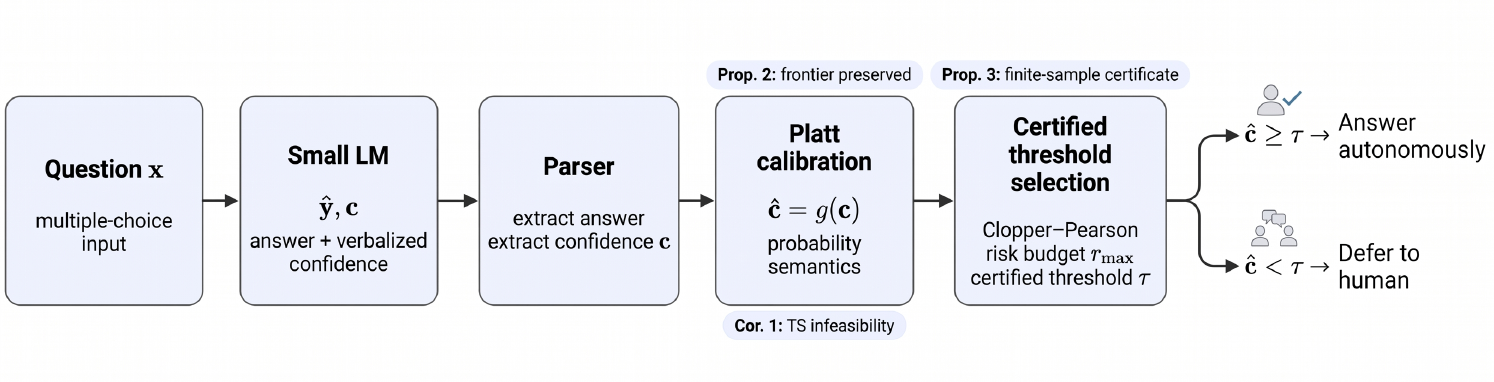}
\caption{Overview: a small LM emits an answer with verbalized
confidence $c$; Platt calibration supplies probability semantics;
a Clopper--Pearson certificate selects the threshold $\tau$; the
policy answers or defers.}
\label{fig:overview}
\end{figure}

\subsection{Problem Formulation}\label{sec:problem}

Let $\mathcal{D}=\{(x_i,y_i)\}_{i=1}^{N}$ be a set of multiple-choice
questions, where $x_i$ is a question with its answer options and
$y_i$ is the correct option label. A language model $f$ maps each
$x_i$ to a predicted label $\hat y_i$ and a raw confidence
$c_i\in[0,1]$, obtained by dividing the verbalized integer score by
100. A calibration map $g:[0,1]\to[0,1]$ produces the calibrated
confidence $\hat c_i=g(c_i)$, with the goal that $\hat c_i$
approximate the probability that $\hat y_i$ is correct. The map $g$
is fit on a held-out calibration set
$\mathcal{D}_{\mathrm{cal}}$, disjoint from the test set on which all
quantities below are reported.

A deferral policy with threshold $\tau\in[0,1]$ answers autonomously
when $\hat c_i\ge\tau$ and defers to a human otherwise. Its empirical
coverage and empirical risk on a set of $N$ examples are
\begin{equation}
\mathrm{Cov}(\tau)=\frac{|\{i:\hat c_i\ge\tau\}|}{N},
\qquad
\mathrm{Risk}(\tau)=\frac{\sum_{i:\hat c_i\ge\tau}
\mathbf{1}[\hat y_i\ne y_i]}{|\{i:\hat c_i\ge\tau\}|},
\label{eq:covrisk}
\end{equation}
with $\mathrm{Risk}(\tau)$ undefined when no example is retained. We
write $R(\tau)=\Pr[\hat y\ne y\mid \hat c\ge\tau]$ for the true
conditional risk under the deployment distribution;
$\mathrm{Risk}(\tau)$ is its finite-sample estimate, and the
certificate of Sect.~\ref{sec:theory} bounds $R(\tau)$ itself. A
deployment specifies a maximum tolerable risk $r_{\max}\in(0,1)$.

\subsection{Verbalized Confidence}\label{sec:extraction}

We elicit $c_i$ directly in text, requesting an answer letter and an
integer score in one fixed format. This requires only generation
access, applies to any model behind a text interface, and costs a
single forward pass per question; token-level probabilities are often
unavailable in deployed settings and are not used. The prompt and the
two-stage parser are described in Sect.~\ref{sec:setup}.

\subsection{Post-hoc Calibration}\label{sec:calibration}

We compare three standard calibration maps, each fit on
$\mathcal{D}_{\mathrm{cal}}$. Let
$\mathrm{logit}(c)=\log\bigl(c/(1-c)\bigr)$ and
$\sigma(z)=1/(1+e^{-z})$ denote the logit and sigmoid functions.
Temperature scaling~\cite{guo2017calibration} applies
\begin{equation}
\hat c=\sigma\!\bigl(\mathrm{logit}(c)/T\bigr),
\label{eq:temp}
\end{equation}
where the calibration temperature $T>0$ is chosen to minimize ECE on
$\mathcal{D}_{\mathrm{cal}}$ over $T\in[0.1,100]$. Platt
scaling~\cite{platt1999} fits a logistic regression
$\hat c=\sigma\bigl(a\,\mathrm{logit}(c)+b\bigr)$ with slope $a$ and
bias $b$; the bias is the one degree of freedom temperature scaling
lacks. The fitted slope is positive in all 22 model--task pairs of our
study, so every fitted Platt map is strictly increasing; the induced
confidence ordering, and with it every threshold-induced retained
set, coincides with that of the raw scores. All guarantees below
that require strict monotonicity therefore apply to the Platt maps
we deploy. Isotonic regression~\cite{zadrozny2002} fits the best
non-decreasing step function from $c$ to correctness. Calibration
quality is measured by the Expected Calibration
Error~\cite{naeini2015}
\begin{equation}
\mathrm{ECE}=\sum_{b=1}^{B}\frac{n_b}{N}\,
\bigl|\mathrm{acc}(b)-\mathrm{conf}(b)\bigr|,
\label{eq:ece}
\end{equation}
where predictions are grouped into $B=10$ equal-width bins by
$\hat c$, bin $b$ holds $n_b$ predictions, and $\mathrm{acc}(b)$ and
$\mathrm{conf}(b)$ are its mean correctness and mean confidence.

\subsection{Threshold Selection}\label{sec:riskcov}

Given $r_{\max}$, the threshold should maximize automation subject to
the risk budget:
\begin{equation}
\tau^\ast=\operatorname*{arg\,max}_{\tau\in[0,1]}\;
\mathrm{Cov}(\tau)
\quad\text{s.t.}\quad
\mathrm{Risk}(\tau)\le r_{\max}.
\label{eq:threshold}
\end{equation}

\begin{proposition}[Minimal feasible threshold]
\label{prop:minfeasible}
Let $\mathcal{F} = \{\tau : \mathrm{Risk}(\tau) \le r_{\max}\}$ be
the feasible set. If $\mathcal{F} \neq \emptyset$, then
$\tau^{*} = \min \mathcal{F}$ solves \eqref{eq:threshold}. If
$\mathcal{F} = \emptyset$, no threshold meets the budget and the
policy returns no autonomous operating point, deferring every
input. The claim does not require $\mathrm{Risk}(\tau)$ to be
monotone; $\mathcal{F}$ may be a union of intervals.
\end{proposition}

\begin{proof}
$\mathrm{Cov}(\tau)$ is a non-increasing step function of $\tau$, so
for any feasible $\tau$,
$\mathrm{Cov}(\min\mathcal{F})\ge\mathrm{Cov}(\tau)$. The minimum is
attained because on a finite sample $\mathrm{Risk}$ is piecewise
constant with finitely many jumps. \qed
\end{proof}

\subsection{Theoretical Properties}\label{sec:theory}

\begin{proposition}[Ranking invariance]\label{prop:invariance}
Let $g$ be strictly increasing and $\hat c_i=g(c_i)$. Then the set of
achievable operating points
$\{(\mathrm{Cov}(\tau),\mathrm{Risk}(\tau)):\tau\in[0,1]\}$ is
identical under $c$ and $\hat c$. In particular temperature scaling
\eqref{eq:temp}, strictly increasing for every $T>0$, leaves the
risk--coverage frontier and the error-detection AUROC unchanged.
\end{proposition}

\begin{proof}
With $g^{-1}(\tau)=\inf\{c:g(c)\ge\tau\}$, strict monotonicity gives
$\{i:g(c_i)\ge\tau\}=\{i:c_i\ge g^{-1}(\tau)\}$. Every retained set
under $\hat c$ is a retained set under $c$ and conversely, and
coverage and risk depend only on the retained set. \qed
\end{proof}

\begin{remark}\label{rem:semantics}
Calibration supplies semantics, not discrimination: after
calibration $\tau=0.9$ reads as retained accuracy near 90\,\%, so
thresholds can be set from domain risk tolerances without labeled
deployment data, though only per task
(Sect.~\ref{sec:ablation}). The proposition also separates the two
calibrators: Platt maps are strictly increasing and preserve every
confidence level, while isotonic ties coarsen the achievable
coverage grid (Sect.~\ref{sec:main}).
\end{remark}

\begin{corollary}[Infeasibility of temperature scaling]
\label{cor:tsbound}
Suppose every raw confidence satisfies $c_i>1/2$ and the model's
accuracy is $\mathrm{acc}<1/2$. Then for every $T>0$ the calibrated
confidences of \eqref{eq:temp} satisfy $\hat c_i>1/2$, and
\begin{equation}
\mathrm{ECE}\;\ge\;\frac{1}{N}\sum_{i=1}^{N}\hat c_i-\mathrm{acc}
\;>\;\tfrac{1}{2}-\mathrm{acc}.
\end{equation}
No temperature can calibrate such a model.
\end{corollary}

\begin{proof}
$c_i>1/2$ gives $z_i=\mathrm{logit}(c_i)>0$, hence
$\sigma(z_i/T)>1/2$ for all $T>0$. By the triangle inequality applied
to \eqref{eq:ece},
$\mathrm{ECE}\ge\bigl|\sum_b \frac{n_b}{N}
(\mathrm{conf}(b)-\mathrm{acc}(b))\bigr|
=\bigl|\overline{\hat c}-\mathrm{acc}\bigr|$, and
$\overline{\hat c}>1/2>\mathrm{acc}$ removes the absolute value. \qed
\end{proof}

\begin{proposition}[Finite-sample risk certificate]\label{prop:cp}
Fix $\tau$ and let the retained calibration examples number $n(\tau)$
with $k(\tau)$ errors among them. If deployment examples are i.i.d.\
from the calibration distribution, then for any $\delta\in(0,1)$,
with probability at least $1-\delta$,
\begin{equation}
R(\tau)\;\le\;
\mathrm{BetaInv}\bigl(1-\delta;\;k(\tau)+1,\;n(\tau)-k(\tau)\bigr),
\label{eq:cp}
\end{equation}
the Clopper--Pearson upper bound. When $\tau$ is selected from the
grid $\mathcal{T}$ of distinct calibration-set confidences, replacing
$\delta$ by $\delta/|\mathcal{T}|$ preserves the guarantee for the
selected threshold
$\hat\tau=\min\{\tau\in\mathcal{T}:
\mathrm{BetaInv}(1-\delta/|\mathcal{T}|;k(\tau)+1,n(\tau)-k(\tau))
\le r_{\max}\}$.
\end{proposition}

\begin{proof}
Conditional on retaining $n(\tau)$ i.i.d.\ examples, $k(\tau)$ is
binomial with success probability $R(\tau)$, and \eqref{eq:cp} is the
exact one-sided binomial upper confidence limit. A union bound over
the $|\mathcal{T}|$ candidate thresholds covers the data-dependent
choice of $\hat\tau$. \qed
\end{proof}

\begin{remark}\label{rem:reuse}
The same calibration set fits the map and selects the threshold;
this does not invalidate the guarantee. The grid $\mathcal{T}$ and
its retained-set family are fixed by the raw-confidence ordering
before any correctness label is inspected, and a strictly monotone
map leaves every retained set unchanged, so the union bound of
Prop.~\ref{prop:cp} covers the label-dependent choice of $\hat\tau$
on either scale. The argument does not extend to isotonic
regression, whose level sets are learned from labels; a second
reason, beyond the coverage coarsening of Sect.~\ref{sec:main}, to
certify with Platt maps only.
\end{remark}

The certificate upgrades the deployment claim from an observed error
rate to a guarantee that holds with confidence $1-\delta$, in the
spirit of selective prediction with guaranteed
risk~\cite{geifman2017}; it assumes no distribution shift between
calibration and deployment, a limitation we return to in
Sect.~\ref{sec:discussion}. Throughout the experiments we use
$\delta=0.05$.

\section{Experiments}\label{sec:experiments}

\subsection{Setup}\label{sec:setup}

\paragraph{Datasets.}
We evaluate on two multiple-choice benchmarks with contrasting
difficulty profiles: ARC-Challenge~\cite{clark2018arc}, grade-school
science questions rewarding knowledge and reasoning, and the
single-answer multiple-choice form of TruthfulQA~\cite{lin2022truthfulqa},
adversarially built so that common misconceptions point toward wrong
answers. Each dataset contributes a 200-question calibration set; the
remaining 1{,}271 and 617 questions form the ARC and TruthfulQA test
sets. Items carry two to five options, so random-guess baselines vary
per item.

\paragraph{Answer-ordering audit.}
In the single-answer multiple-choice form of TruthfulQA used by our
pipeline, all 817 pooled calibration and test items place the
correct option in the first position, so any letter-scored protocol
rewards positional preference: a constant first-letter policy
attains perfect accuracy. Before any inference we permute the
options of every question with a deterministic shuffle seeded by a
hash of the question text, and normalize 26 ARC items with numeric
labels. After repair the correct-letter distribution is near uniform
(A/B/C/D $=$ 203/223/200/191) and the constant-first policy drops to
24.8\,\%, close to chance. We recommend this audit as a simple
safeguard for letter-scored uses of TruthfulQA.

\paragraph{Models.}
We study eleven open-weight instruction-tuned models across three
families and a 28-fold size range: Qwen2.5 at 0.5B to 14B, Llama~3.2 at
1B and 3B with Llama~3.1 at 8B, and Gemma~3 at 1B to 12B. All run
locally through Ollama in default quantized builds on one MacBook Air
M4 with 32\,GB of memory, using sampling temperature zero, distinct
from the calibration temperature $T$ of Sect.~\ref{sec:calibration}.
The study totals 25{,}168 predictions.

\paragraph{Elicitation, parsing, and metrics.}
A single prompt requests an answer letter and an integer confidence
in a fixed letter-comma-number format. Responses pass a strict
parser and, on failure, a conservative lenient parser that recovers
a valid option letter and, when unambiguous, a confidence value;
refusals are never recovered. Predictions with a recoverable
confidence enter all calibration and deferral analyses; answer-only
recoveries count toward accuracy but not calibration; unrecoverable
responses are treated as automatic deferrals in deployment. Ten of
eleven models exceed 96\,\% strict compliance and 99\,\% usable
confidence on every split, while Llama3.2-1B falls to 27.7\,\%
strict compliance yet retains 71.6\,\% usable confidence, so format
failure and signal absence are distinct failure modes. All calibrators are fit on the
calibration set and evaluated once on the test set. We report
accuracy, Expected Calibration Error (ECE) with ten equal-width
bins, AUROC for error detection, and coverage at certified risk;
formal definitions appear in Sect.~\ref{sec:problem}.

\subsection{Main Results}\label{sec:main}

\begin{table}[t]
\caption{Accuracy, ECE under four calibration treatments,
error-detection AUROC, and certified coverage at $r_{\max}=20\%$
under Platt scaling (Prop.~\ref{prop:cp}). Best ECE per row in
bold; $\dagger$: temperature scaling at the search bound
$T^{*}\!\ge\!100$ (Cor.~\ref{cor:tsbound}); --: no feasible
threshold, a refusal.}
\label{tab:main}
\centering\small
\setlength{\tabcolsep}{3.5pt}
\begin{tabular}{llccccccc}
\hline
\multirow{2}{*}{Model} & \multirow{2}{*}{Data} &
\multirow{2}{*}{Acc.} & \multicolumn{4}{c}{ECE} &
\multirow{2}{*}{AUROC} & \multirow{2}{*}{CertCov (\%)} \\
\cmidrule(lr){4-7}
 & & & uncal. & TS & Platt & Iso & & \\
\hline
\multirow{2}{*}{Qwen2.5-0.5B} & ARC & 33.7 & 0.497 & 0.168$^{\dagger}$ & 0.022 & \textbf{0.019} & 0.522 & -- \\
 & TQA & 22.4 & 0.631 & 0.283$^{\dagger}$ & 0.060 & \textbf{0.050} & 0.477 & -- \\
\hline
\multirow{2}{*}{Qwen2.5-1.5B} & ARC & 66.5 & 0.210 & 0.139 & 0.091 & \textbf{0.084} & 0.492 & -- \\
 & TQA & 42.9 & 0.412 & 0.101$^{\dagger}$ & 0.116 & \textbf{0.071} & 0.495 & -- \\
\hline
\multirow{2}{*}{Qwen2.5-3B} & ARC & 75.0 & 0.229 & \textbf{0.014} & 0.015 & 0.014 & 0.554 & -- \\
 & TQA & 59.0 & 0.389 & \textbf{0.052} & 0.059 & 0.064 & 0.537 & -- \\
\hline
\multirow{2}{*}{Qwen2.5-7B} & ARC & 87.9 & 0.043 & 0.038 & 0.026 & \textbf{0.024} & 0.579 & 93.3 \\
 & TQA & 69.2 & 0.211 & 0.080 & \textbf{0.030} & 0.042 & 0.581 & -- \\
\hline
\multirow{2}{*}{Qwen2.5-14B} & ARC & 91.8 & \textbf{0.023} & 0.033 & 0.033 & 0.028 & 0.655 & 99.8 \\
 & TQA & 78.5 & 0.129 & 0.119 & \textbf{0.071} & 0.101 & 0.629 & 43.6 \\
\hline
\multirow{2}{*}{Llama3.2-1B} & ARC & 26.4 & 0.546 & 0.240$^{\dagger}$ & \textbf{0.070} & 0.070 & 0.531 & -- \\
 & TQA & 24.8 & 0.539 & 0.254$^{\dagger}$ & \textbf{0.073} & 0.082 & 0.503 & -- \\
\hline
\multirow{2}{*}{Llama3.2-3B} & ARC & 63.5 & 0.240 & 0.110 & \textbf{0.040} & 0.049 & 0.533 & -- \\
 & TQA & 41.6 & 0.425 & 0.089$^{\dagger}$ & \textbf{0.036} & 0.036 & 0.468 & -- \\
\hline
\multirow{2}{*}{Llama3.1-8B} & ARC & 75.8 & 0.164 & 0.114 & \textbf{0.027} & 0.035 & 0.538 & -- \\
 & TQA & 50.7 & 0.401 & 0.055 & 0.067 & \textbf{0.067} & 0.519 & -- \\
\hline
\multirow{2}{*}{Gemma3-1B} & ARC & 25.5 & 0.687 & 0.262$^{\dagger}$ & 0.030 & \textbf{0.028} & 0.516 & -- \\
 & TQA & 23.7 & 0.675 & 0.278$^{\dagger}$ & 0.088 & \textbf{0.053} & 0.489 & -- \\
\hline
\multirow{2}{*}{Gemma3-4B} & ARC & 72.6 & 0.224 & 0.038 & 0.031 & \textbf{0.023} & 0.544 & -- \\
 & TQA & 41.2 & 0.541 & 0.104 & \textbf{0.091} & 0.096 & 0.527 & -- \\
\hline
\multirow{2}{*}{Gemma3-12B} & ARC & 89.2 & 0.068 & 0.050 & 0.051 & \textbf{0.048} & 0.509 & -- \\
 & TQA & 67.1 & 0.262 & 0.066 & 0.085 & \textbf{0.028} & 0.640 & -- \\
\hline
\end{tabular}
\end{table}

Table~\ref{tab:main} reports all eleven models on both tasks. Platt
scaling and isotonic regression improve calibration in every row and
dominate temperature scaling throughout, most sharply on small
models: Gemma3-1B on ARC falls from an uncalibrated ECE of 0.687 to
0.262 under TS but to 0.030 under Platt. The TS failures are
structural, not a matter of tuning. Eight model--task pairs drive
the optimal temperature to the search bound, and their residual ECE
sits at the floor predicted by Cor.~\ref{cor:tsbound}: Qwen2.5-0.5B
on TruthfulQA reaches 0.283 against a bound of 0.276, and
Llama3.2-1B 0.254 against 0.252. A single temperature cannot push
confidence below one half, and these models pair near-ceiling verbal
confidence with accuracy near chance; Platt's bias term supplies
exactly the missing degree of freedom.

\paragraph{Platt versus isotonic.}
On ECE the two are close and isotonic wins several rows; the
difference lies in what they preserve. As
Prop.~\ref{prop:invariance} predicts, the strictly increasing Platt
maps retain every distinct confidence level in all 22 model--task
pairs, while isotonic regression collapses an average of 20.2 levels
to 5.3, and to a single level in the worst cases, coarsening the
coverage grid on which threshold-based deferral operates. We
therefore adopt Platt scaling for all deferral experiments.

\paragraph{Discrimination.}
AUROC for error detection is modest everywhere: models below 2B sit
near 0.5, Llama3.2-3B on TruthfulQA falls to 0.468, and the largest
models reach only 0.63 to 0.66. High accuracy does not imply
discrimination, Gemma3-12B attains 89.2\,\% on ARC with an AUROC of
0.509, so verbalized confidence alone is too weak an error signal to
act on directly, motivating the calibrated, certified deferral of
Sect.~\ref{sec:cert}.

\subsection{Scaling and Task Dependence}\label{sec:scaling}

\begin{figure}[t]
\centering
\includegraphics[width=0.80\textwidth]{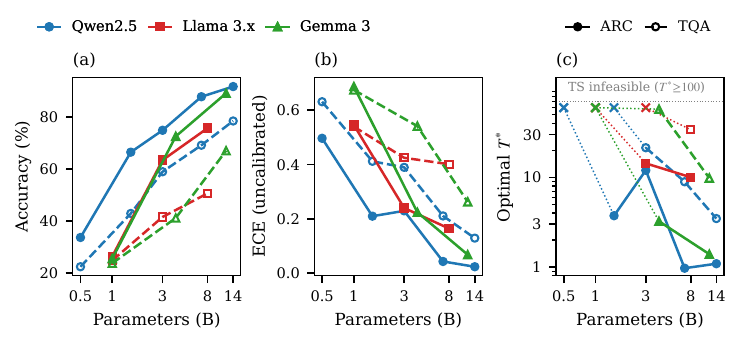}
\caption{Accuracy, uncalibrated ECE, and optimal temperature $T^{*}$
across model size. Solid: ARC; dashed with open markers: TruthfulQA.
In panel~(c), $\times$ marks pairs at the search bound
$T^{*}\!\ge\!100$, joined by dotted connectors to the feasible
segment of the same curve; isolated markers are the only feasible
size in that family--task pair.}
\label{fig:scaling}
\end{figure}

\paragraph{Size and task.}
Accuracy rises monotonically with size within every family
(Fig.~\ref{fig:scaling}a) and uncalibrated ECE falls in step
(Fig.~\ref{fig:scaling}b). TruthfulQA pulls sub-1B models below
their chance baselines, Qwen2.5-0.5B to 22.4\,\% against roughly
26\,\%, confirming that its distractors actively mislead weak
models. At every size the TruthfulQA curve requires stronger
correction than the ARC curve (Fig.~\ref{fig:scaling}c): for
Qwen2.5-7B the optimal temperature is 0.97 on ARC and 8.94 on
TruthfulQA, for Gemma3-4B 3.24 against 57.8. Task type, not model identity alone, sets the
calibration requirement.

\paragraph{Mechanism.}
Median verbalized confidence is nearly constant across sizes,
between 0.80 and 1.00 in every model--task pair. Calibration
improves with scale because accuracy climbs toward a fixed
confidence level, not because larger models hedge; this single fact
explains the mirror symmetry of panels~(a) and~(b).

\subsection{Certified Deferral}\label{sec:cert}

\begin{figure}[t]
\centering
\includegraphics[width=0.75\textwidth]{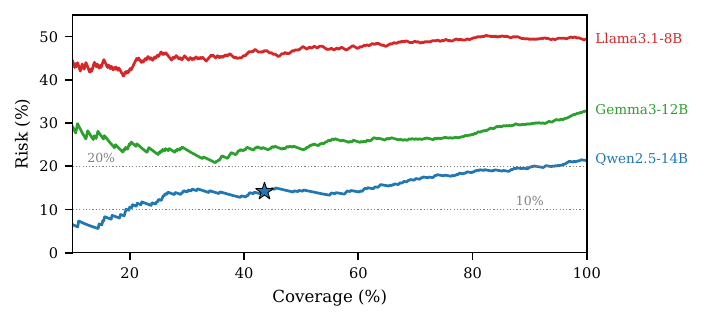}
\caption{Risk--coverage curves on TruthfulQA after Platt calibration
(coverage above 10\,\%). Dotted lines: 10\,\% and 20\,\% risk
budgets. The star is the Qwen2.5-14B operating point certified by
Prop.~\ref{prop:cp} at $r_{\max}=20\%$, plotted at realized test
coverage and risk; Gemma3-12B and Llama3.1-8B receive no certified
threshold.}
\label{fig:riskcov}
\end{figure}

\paragraph{Certification.}
The Clopper--Pearson procedure of Prop.~\ref{prop:cp} grants autonomy
sparingly. At $r_{\max}=20\%$ it certifies Qwen2.5-14B on ARC at
99.8\,\% coverage with a realized test risk of 8.1\,\%, Qwen2.5-7B on
ARC at 93.3\,\% with 10.8\,\%, and Qwen2.5-14B on TruthfulQA at
43.6\,\% with 14.2\,\% (Fig.~\ref{fig:riskcov}); every other pair
receives no certified threshold. At $r_{\max}=10\%$ no pair is certified under the observed
calibration errors and the union-corrected Clopper--Pearson bound,
illustrating the sample-size cost of strict finite-sample
guarantees. Refusal here is an informative output, not a failure mode; we return
to this in Sect.~\ref{sec:discussion}.

\subsection{Ablations}
\label{sec:ablation}

\begin{table}[t]
\caption{Threshold transfer from ARC to TruthfulQA at
$r_{\max}=20\%$. Violation: realized TruthfulQA risk minus
$r_{\max}$, in percentage points. TS uses per-task temperatures.
0-cov: the transferred threshold retains no test example.}
\label{tab:transfer}
\centering\small
\begin{tabular}{lcccc}
\hline
\multirow{2}{*}{Model} & \multicolumn{2}{c}{Uncalibrated} &
\multicolumn{2}{c}{Calibrated (TS)} \\
\cmidrule(lr){2-3}\cmidrule(lr){4-5}
 & Cov.\,(\%) & Viol.\,(pp) & Cov.\,(\%) & Viol.\,(pp) \\
\hline
Qwen2.5-7B & 99.8 & $+10.8$ & 100.0 & $+10.8$ \\
Qwen2.5-14B & 99.8 & $+1.3$ & 100.0 & $+1.5$ \\
Llama3.1-8B & 95.9 & $+29.6$ & 97.9 & $+29.9$ \\
Gemma3-4B & 14.4 & $+30.6$ & 0.0 & 0-cov \\
Gemma3-12B & 99.4 & $+12.6$ & 34.8 & $+0.9$ \\
\hline
\end{tabular}
\end{table}

\paragraph{Ablating calibration granularity.}
On the 14 TS-feasible pairs, a single global temperature gives a
mean test ECE of 0.157, per-task temperatures 0.155, and
per-model-per-task temperatures 0.072. Knowing the task alone buys
almost nothing; knowing the model halves the error. No reliable
shared task-level correction emerges in this setting, so calibration
must be fit per deployment pair.

\paragraph{Ablating the per-task certificate.}
Reusing ARC thresholds on TruthfulQA violates the 20\,\% budget for
every tested model (Table~\ref{tab:transfer}), by 1.3 to 30.6
percentage points before calibration. Per-task calibration repairs
scale mismatch, Gemma3-12B falls from $+12.6$ to $+0.9$ points, but
cannot repair weak discrimination: Llama3.1-8B stays near $+30$, as
Prop.~\ref{prop:invariance} requires, and Gemma3-4B's transferred
threshold retains no examples at all. Safe deployment needs the
per-pair certificates of Sect.~\ref{sec:cert}, not thresholds shared
across tasks or models.

\paragraph{Ablating the elicitation format.}
Re-running three models (Gemma3-1B, Llama3.1-8B, Qwen2.5-14B) on
TruthfulQA with a reworded prompt shifts accuracy by at most four
percentage points and preserves every qualitative finding: median
raw confidence stays at 0.90 to 0.95 against accuracies of 27\,\%
to 78\,\%, uncalibrated ECE remains severe, post-hoc calibration
restores test ECE to at most 0.09, AUROC stays in the weak 0.54 to
0.58 band, and temperature scaling for the 1B model again diverges
to the search bound of Cor.~\ref{cor:tsbound}. Verbalized
overconfidence is a property of the models, not of one prompt.

\subsection{Transparency and Deployment}\label{sec:transparency}

\paragraph{Auditable decisions.}
Every deferral decision reduces to the comparison $\hat c\ge\tau$
against a certified threshold, so each case exposes its
justification as $(c,\hat c,\tau,\mathrm{action})$. For Qwen2.5-14B
on TruthfulQA at the certified $\tau=0.702$, a correct answer with
raw confidence $c=1.00$ is calibrated to $\hat c=0.94$ and answered
autonomously, while a domestic-violence legal question with $c=0.90$
is lowered to $\hat c=0.696<\tau$ and deferred; the model's answer
was in fact wrong. Calibration turns a naive high-confidence answer
into an auditable deferral.

\paragraph{Deployment profile.}
All models ran locally on a Mac M4 machine, with quantized
sizes from 0.4 to 9.0\,GB on disk and mean latency from 0.14 to
2.55\,s per question. Ten of eleven models exceeded 96\,\% strict
format compliance and 99\,\% usable-confidence rate; the exception,
Llama3.2-1B, retained 71.6\,\% usable confidence despite 27.7\,\%
strict compliance, so format failure and confidence absence are
distinct deployment failure modes.

\section{Discussion, Limitations, and Conclusion}\label{sec:discussion}

\paragraph{What calibration buys, and what it cannot.}
By Prop.~\ref{prop:invariance}, calibration cannot improve the
risk--coverage frontier or rescue a model whose confidence carries
little error signal, as Llama3.1-8B shows under transfer. Its value
is semantic: after per-task calibration a threshold reads as a
target accuracy, and the certificate of Prop.~\ref{prop:cp} converts
that reading into a guarantee. Calibration supplies meaning, the
frontier bounds what is achievable, the certificate determines what
is safe.

\paragraph{Refusal as a feature.}
At $r_{\max}=10\%$ the certificate authorizes no model, at 20\,\%
only the strongest Qwen pairs. A system that always returns an
operating point risks deploying models that should not act
autonomously; one that refuses makes the missing capability visible
and returns the decision to the human, the correct default in
high-stakes settings.

\paragraph{Limitations.}
Five limitations bound our claims: both benchmarks are multiple
choice, so open-ended generation remains untested; the certificate
assumes i.i.d.\ deployment and is void under distribution shift; the
200-question calibration sets cap the certifiable regime, as the
empty $r_{\max}=10\%$ outcome shows; models are capped at 14B under
one quantization scheme; and the format-robustness check covers two
elicitation formats on one dataset and three models.

\paragraph{Conclusion.}
We presented a risk-controlled deferral framework coupling
verbalized confidence, post-hoc calibration, and certified threshold
selection, characterized across eleven small open-weight models on
two contrasting benchmarks. Theory and measurement agree, from the
temperature-scaling infeasibility floor met within a percentage
point to the failure of cross-task threshold transfer, and the study
runs on a single consumer laptop. At strict budgets the framework
authorizes autonomy only where evidence supports it; we regard such
calibrated refusal as the correct default for deploying small
language models in decision-critical settings.

\begin{credits}
\subsubsection{\discintname}
The authors have no competing interests to declare that are relevant to the content of this article.
\end{credits}

%
%
%
\bibliographystyle{splncs04}
\bibliography{references}

\end{document}